\def\year{2020}\relax
\documentclass[letterpaper]{article} %DO NOT CHANGE THIS
\usepackage{aaai20}  %Required
\usepackage{times}  %Required
\usepackage{helvet}  %Required
\usepackage{courier}  %Required
\usepackage{url}  %Required
\usepackage{graphicx}  %Required
\usepackage{listings}%

\usepackage[dvipsnames]{xcolor}
\usepackage{amsopn}
\usepackage{amsmath}
\usepackage{amsfonts}
\usepackage{amssymb}
\usepackage{algorithm}
\usepackage[noend]{algpseudocode}
\usepackage{subcaption}
\usepackage{booktabs}
\usepackage{multirow}
\usepackage{amsthm}
\usepackage{textcomp}
\usepackage{bm}
 \usepackage[font={small}]{caption}
\newcommand{\exact}{Expectation-Aware}
\newtheorem{theorem}{Theorem}
\theoremstyle{definition}
\newtheorem{defn}{Definition} % definition numbers are dependent on theorem numbers
 
\nocopyright

\makeatletter
\def\BState{\State\hskip-\ALG@thistlm}
\makeatother
\begin{document}
\newcommand{\sscomment}[1]{\textcolor{red}{[Sid] #1}}
\newcommand*{\affaddr}[1]{#1} % No op here. Customize it for different styles.
\newcommand*{\affmark}[1][*]{\textsuperscript{#1}}

%\title{Hierarchical Expertise-Level Modeling for User Specific Robot-Behavior Explanations}
%\title{Plan Explanation Through Search in an Abstract Model Space}
%\title{Model Reconciliation Revisited:\\ Moving Beyond Soliloquy in the Absence of Pre-specified User Models}
\title{\exact~Planning: A Unifying Framework for Synthesizing \\and Executing Self-Explaining Plans for Human-Aware Planning}
% * <sidsrivast@gmail.com> 2017-11-17T17:58:33.925Z:
% 
% > Explanations over hierarchy of model abstractions
% Needs work... We need to say something about the process of explaining  autonomous behavior by interactively concretizing user understanding/knowledge/models
% 
% explaining autonomous behavior/plans through interactive model assessment and concretization
% 
% 
% Searching for  explanations through interactive knowledge assessment and refinement
% 
% Plan explanation as search  in a hierarchical model space
% 
% 
% 
% ^.
%\author{Paper ID 1813\\Anonymized for Submission}
\author{
Sarath Sreedharan\affmark[1],
Tathagata Chakraborti\affmark[2],
Christian Muise\affmark[2]
\and Subbarao Kambhampati\affmark[1]\\
\affaddr{\normalfont{\affmark[1]CIDSE,
Arizona State University, Tempe, AZ 85281 USA}\\
\affaddr{\affmark[2]IBM Research AI,
Cambridge, MA, USA}\\
ssreedh3@asu.edu, tathagata.chakraborti1@ibm.com, christian.muise@ibm.com, rao@asu.edu
%Anonymous
}}

\maketitle
\begin{abstract}
In this work, we present a new planning formalism called 
{\em \exact~}planning for decision making with humans in the loop
where the human's expectations about an agent may differ 
from the agent's own model. 
We show how this formulation allows agents to not only leverage existing 
strategies for handling model differences but can also exhibit novel behaviors 
that are generated through the combination of these different strategies.
Our formulation also reveals a deep connection to existing 
approaches in epistemic planning.
Specifically, we show how we can leverage classical planning compilations 
for epistemic planning to solve {\em \exact~}planning problems.
To the best of our knowledge, the proposed formulation is the first complete 
solution to decision-making in the presence of diverging user expectations that 
is amenable to a classical planning compilation while successfully combining 
previous works on explanation and explicability.
We empirically show how our approach provides a computational advantage over existing approximate approaches that unnecessarily try to search in
the space of models while also failing to facilitate the full gamut of
behaviors enabled by our framework.
\end{abstract}

\section{Introduction}
One of the greatest challenges in designing agents that can work with humans
is in making sure that the agents are capable of acting in a manner that is interpretable to the humans.
A major barrier towards achieving fluent collaboration occurs when
the human's expectations regarding the agent's capabilities and preferences 
differ from reality.
Such knowledge asymmetry implies that 
% even in cases where the human teammate is a passive observer, the agent can no longer solely rely on their individual model to come up with their plans.
even in cases where the agent is coming up with the best decisions it can, 
the human would not be able to agree to the quality of that plan.
Previous works have proposed two strategies to handle this: (1) provide information that reconciles the model differences, either through explicit communication \cite{explain} or by performing actions that convey robots capabilities \cite{kwon2018expressing} (2) or by acting in a manner that aligns with human expectations \cite{exp-yu}.

While each of these are reasonable strategies on their own, for the agent to be truly effective we would want it to be capable of combining the strengths of each. Unfortunately, current approaches to combining these methods (for example \cite{balance}) have generally fallen short in the kind of behavior they can generate. We are unaware of any existing works that is truly able to capitalize on the agent's ability to effect and leverage human expectations through explicit communication and behavior.
%been incomplete in the types of behavior it can generate and generally tend to overlook one intuitive aspect of the problem setting, namely the fact that agent's task level actions could influence the observer's expectations. For example, an agent does not need to explicitly communicate that a door is unlocked to convince the observer when the simple act of opening it would suffice.

Our formulation, on the other hand, leads to what may be best described as {\em self-explaining plans} with the plan now containing actions that are responsible for explaining the rest of the plan. 
Such explanations may be delivered by purely communicative actions (thereby capturing \cite{explain}) that are meant to update the human's mental model or task level actions that could also have epistemic side effects (thereby allowing for actions of the type studied in \cite{kwon2018expressing}). Additionally, the framework allows for selecting plans that aligns with human expectations whenever possible.
Our contributions are thus two-fold:

\begin{itemize}
\item[-] 
We present the first unification of various threads of planning with differing human expectation: including acting in-accordance with the human expectation (explicability), bridging model asymmetry through implicit (epistemic effects of plan execution on the mental model) and explicit communication (explanations).
%\item[-] 
%We show that unlike previous approaches, our algorithm is complete, while
%also allowing newer types of behaviors (e.g. accounting for the 
%effect of execution on the mental model) previously unachievable in
%existing algorithms.
\item[-] 
We show how our formulation is complete (unlike previous approaches) while also lending itself to a compilation to classical planning problems. The latter provides significant computational advantage with respect to existing algorithms in this space that search directly in the space of models.
\end{itemize}

% The rest of paper is structured as follows, in Section \ref{backg} we will cover some of the background material for the current work and the setting we will be studying, next in Section \ref{usar} we will go over the illustrative example we will be using through out the paper and in Section \ref{ea-plan} we will go over our proposed planning framework and a possible way to generate solutions for it. Section \ref{balance} will introduce a special class of solutions within this framework that can account for user's expectations about the agent's optimality. Section \ref{emp} will cover evaluation of the system including a demonstration of our system ability to generate solutions of varying properties and an emipirical evaluation of the work.
% Finally Section \ref{related} will discuss some related works, including earlier works that have looked at the problem of combining model reconciliation and explicability and we will conclude the paper in Section \ref{concl} with a small discussion of future works.

\section{Background}
\label{backg}
We will assume that the planning models used by both the human and the robot are represented as classical planning problems 
described by the tuple $\mathcal{M} = \langle F, A, I, G, C\rangle$ \cite{geffner2013concise}, 
where $F$ is the set of propositional fluents used to describe the planning task states, $A$ the set of actions, $I$ the initial state, $G$ the goal.
% and \alert{$C$ the cost of the actions}. 
Each action $a \in A$ is further defined as a tuple
$a = \langle \textrm{prec}^a, \textrm{adds}^a, \textrm{dels}^a\rangle$, where $\textrm{prec}^a$ is its preconditions, and $\textrm{adds}^a$ and $\textrm{dels}^a$ are its add and delete effects. 
The precondition is a propositional formula defined over state fluents such that an action $a$ can only be executed in a state $S$ if $S \models \textrm{prec}^a$. 
The effects are generally of the form $c \rightarrow e$, where the antecedent represents the condition under which the effect $e$ should be applied (where the fluent corresponding to $e$ is set to true in the state if $c \rightarrow e$ is part of the add effects, and if it is part of the delete it is set to false). 

Each action is associated with a cost $C(a)$. 
A plan or a sequence of actions $\pi = \langle a_1, ..., a_n\rangle$ 
is a valid solution of a planning problem $\mathcal{M}$ if 
$\pi(I)\models_{\mathcal{M}} G$ and $G \subseteq \pi(I)$. 
The cost of a plan is the sum of individual action costs, i.e. 
$C(\pi) = \sum_{i=1}^n C(a_i)$. 
A plan $\pi$ is said to be optimal if there exist no valid plan $\pi'$ such that $C(\pi') < C_{\mathcal{M}}(\pi)$. 
We will use $\Pi^*_{\mathcal{M}}$ to represent the set of all optimal plans 
for $\mathcal{M}$.

%\note{Parameterized representation....}

The particular setting we are interested in involves an agent that 
makes decisions using its own model $\mathcal{M}_R = \langle F,A_R,I_R,G_R,C\rangle$ while a human
% (an observer) 
evaluates the plan using their mental model $\mathcal{M}_H = \langle F,A_H,I_H,G_H,C\rangle$.
For ease of discussion, we concentrate on the specific case where conditions for actions only consist of conjunction of positive literals and the agents have the same cost.
While the human is under the assumption that $M_H$ is an accurate representation of the task at hand, the model could be different from $\mathcal{M}_R$ in terms of action definitions, the initial state, and the goal. 
This difference means that plans generated for the model $\mathcal{M}_R$ may have different properties in the mental model $\mathcal{M}_H$.
For example, a plan $\pi^*$ that is optimal in $\mathcal{M}_R$ may be considered suboptimal or even in-executable by the human. 

When model asymmetry becomes a source of confusion for the observer, 
explaining the plan must involve bridging this gap.
One of the ideas proposed by earlier works in {\em explanations as model reconciliation} (c.f \cite{explain}) is that given a specific plan, the agent does not need to achieve complete reconciliation.  
Rather they can focus on providing enough information that the current plan has required properties (such as executability, optimality, etc.).
When the agent is aware of $M_H$, it can use this knowledge to figure out the minimal (where minimality of explanations defined with respect to an explanation cost $C_E$) information it needs to provide to achieve the required properties. For example, the problem of identifying explanations for establishing optimality of a given plan $\pi$ thus becomes:
{
\[\textrm{argmin}_{\mathcal{E}}(C_E(\mathcal{E}))\]
\[\textrm{such that } \pi \in \Pi^*_{\mathcal{M}_H + \mathcal{E}}\]
}
\noindent where $\mathcal{E}$ is a set of model information about the agent to be provided to the user as explanation (this could include truth value of fluents in initial state, presence or absence of literals in preconditions/effects, etc.) and $\mathcal{M}_H + \mathcal{E}$ is the updated user model after the explanation. 
Note that our use of `+' operator does not imply that all model reconciliation explanations are additive as $\mathcal{E}$ could include information aimed at correcting user's misconceptions about additional effects or even additional actions that the robot is capable of. We will follow the conventions set by \cite{explain} and will focus on three main types of model updates:
%\begin{enumerate}

\vspace{5pt}
\noindent(1) Turn a fluent p true or false in initial state (represented by the operator $\{\textrm{add}/\textrm{remove}\}\textrm{-p-from-I}$)

\vspace{5pt}
\noindent(2) Add or remove a fluent p from the precondition (also add or delete effect) list of an action a (represented by the operator $\{\textrm{add}/\textrm{remove}\}\textrm{-p-from-\{prec/adds/dels\}-of-a}$)

\vspace{5pt}
\noindent(3) Add or remove a fluent p from the goal list (represented by the operator $\{\textrm{add}/\textrm{remove}\}\textrm{-p-from-G}$)

%\end{enumerate}
% explained
\vspace{5pt}

This focuses on cases where the agent is explaining its plan
to the human after generating it. 
The flip side would be to try generating plans that are tailored for the human model. This is referred to as explicable planning and the most basic version of this problem can be formulated as:
\[\textrm{argmin}_{\pi}(C(\pi))\]
\[\textrm{Such that } \pi(I_R) \models G_R \textrm{ and } \pi(I_H) \models G_H\]
This computes a plan that is executable in the agent and human mental model with the lowest cost. 
Our approach is capable of both explaining its plans as well as choosing plans that align with the user expectations. Before delving into details of the formulation, let us introduce the search and rescue domain (a variation from \cite{balance}) which we will use as an illustrative example for the rest of the paper.

\section{Running Example: Search \& Rescue}

\label{usar}
\begin{figure*}[!th]
\centering
\includegraphics[scale=0.5]{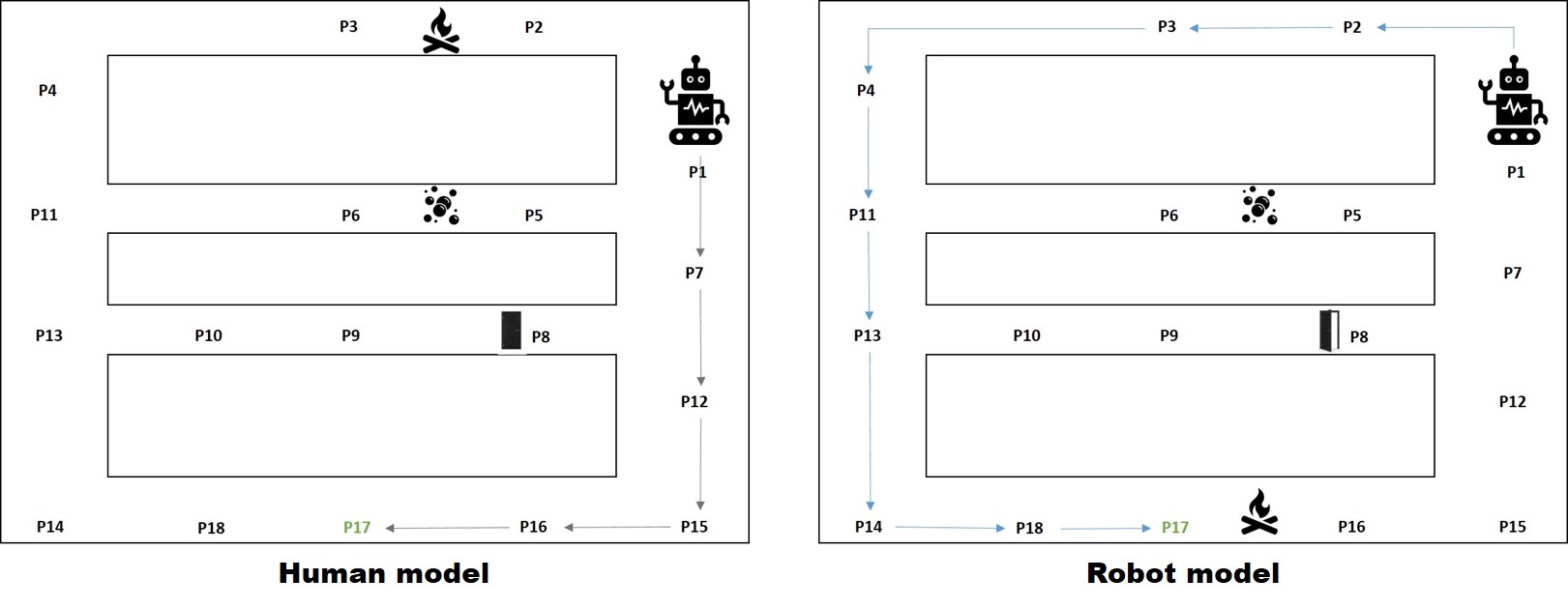}
    \caption{\small { The basic robot and human maps. The robot starts at P1 and needs to go to P17. The human incorrectly believes that the path from P16 to P17 is clear and the one from P2 to P3 is blocked by fire. Both agents know that there are some movable rubble between p5 and P6 that can be cleared with the help of a costly clear\_passage action. Finally in the human model the door at P8 is locked while it is unlocked in the robot model and robot can't open unlocked doors.}}
\label{fig:1}
%\vspace{-10pt}
\end{figure*}

A typical Urban Search and Rescue (USAR) 
scenario consists of an autonomous robot deployed to a disaster scene with an external commander who is monitoring its activities. 
Both agents start with the same model of the world (i.e the map of the building before the disaster) but the models diverge over time since the robot, being internal to the scene, has access to updated information about the building.
This model divergence could lead to the commander incorrectly evaluating valid plans from the robot as sub-optimal or even unsafe. One way to satisfy the commander would be to communicate or explain changes to the model that led the robot to come up with those plans in the first place.

% presents an ideal testbed for research on \exact~(EA) planning as it looks at cases where the decision to follow suboptimal or in-executable plan can be potentially disastrous, yet limitations in communications capability could prevent the agents from providing detailed explanations. 

Figure \ref{fig:1} illustrates a scenario where the robot needs to travel from P1 to its goal at P17. 
The optimal plan expected by the commander is highlighted in grey in their map
and involves the robot moving through waypoint P7 and follow that corridor to go to P15 and then finally to P16. The robot knows that it should in fact be moving to P2 -- its optimal plan is highlighted in blue. 
This disagreement rises from the fact that the human incorrectly believes that the path from P16 to P17 is clear while that from P2 to P3 is blocked.

If the robot were to follow the explanation scheme established in \cite{explain}, it would stick to its own plan and provide the following explanation:

{\small \begin{lstlisting}[mathescape]
> remove$\textrm{-}$(clear p16 p17)$\textrm{-}$from$\textrm{-}$I
    (i.e. Path from P16 to P17 is blocked)
> add$\textrm{-}$(clear p2 p3)$\textrm{-}$to$\textrm{-}$I 
    (i.e. Path from P2 to P3 is clear)
\end{lstlisting}}

If the robot were to stick to a purely explicable plan \cite{exp-yu} then it can choose to use the passage through P5 and P6 after performing a costly clear\_passage action (this plan is not optimal in either of the models).

\section{Expectation-Aware Planning}
\label{ea-plan}
We call the task of computing plans with the expectations of an external agent: \exact~planning. 
% In this section, we will formally define an \exact~ problem. 
% and establish the basic notions for valid solutions for a given \exact~ problem.

\begin{defn}
{\em A {\em \exact~ planning problem} (EA) is defined by the tuple $\Psi = \langle \mathcal{M}_R, \mathcal{M}_H\rangle$, where $\mathcal{M}_R$ is the robot model and $\mathcal{M}_H$ is the model ascribed to the robot by an observer. A solution to the problem $\Psi$ is then given by the tuple $\langle \mathcal{E}^{\Psi}, \pi_{\Psi}\rangle$, where $\mathcal{E}^{\Psi}$ is set of model updates for $\mathcal{M}_H$ consistent with $\mathcal{M}_R$ and $\pi_{\Psi}$ a plan. The given solution is considered valid iff
$\pi_{\Psi} (I_{\mathcal{M}_R}) \models_{\mathcal{M}_R} G_{\mathcal{M}_R}$
(it is valid in the agent model) as well as $\pi_{\Psi} (I_{\mathcal{M}_H + \mathcal{E}^{\Psi})}) \models_{\mathcal{M}_H + \mathcal{E}^{\Psi}} G_{\mathcal{M}_H + \mathcal{E}^{\Psi}}$ (valid in the updated mental model).}
\end{defn}
This means that a solution to an expectation aware problem may consist of model information to be provided to the observer along with the plan that needs to be followed by the agent. 
% Moreover a given solution is considered valid if and only if the plan chosen by the agent is valid in the resultant model obtained by applying the current set of model updates. 
In the USAR example, the optimal robot plan along with the two initial state updates, and the explicable plan with no model updates, would both be valid solutions.
% for the EA problem specified by the example. 
%Also note that the problem definition itself makes no assumptions about how and when the model information are actually presented to the user. So a system that delivers the inf

While at first glance, the need to keep track of both models and identifying the model changes may make the problem of solving EA planning problems considerably harder than the original decision making problem. 
However, we show that, in fact, finding a valid solution in this setting is no harder than identifying valid plans for classical planning problems:

\begin{theorem}
For a given EA problem $\Psi = \langle \mathcal{M}_R, \mathcal{M}_H\rangle$, where both $ \mathcal{M}_R$ and $\mathcal{M}_H$ are represented as classical planning problems, the problem of identifying a valid solution for $\Psi$ is PSPACE-complete.
\end{theorem}
\begin{proof}[Proof Sketch]
The PSPACE-hardness of an EA is easy to establish since the problem of planning with just agent model can be mapped to a specific EA planning scenario where both agent and user have the same model. We can establish membership in PSPACE class by showing that there exist a sound and complete compilation from EA to a planning problem with conditional effects and disjunctive/negative preconditions that is linear in size of the original planning problems. 
We can then follow the same proof specified in \cite{bylander} to show that the problem of plan existence is still in PSPACE for this class of planning problems. 
The exact details of the compilation along with the soundness and completeness proofs will be discussed in Section \ref{compl}.
\end{proof}

% Now that we have established the computational properties of the problem, we will focus on efficient methods for generating solutions.
% In particular, we look to leverage modern planning heuristics by engineering a compilation of EA into a classical planning problem.

\subsection{Self-Explaining Plans as Solutions to EA}
%\subsection{Self-Explaining Plans as Solutions to {\em Expectation-Aware} Planning Problems}
\label{compl}
One of the main challenges of compiling an EA problem to a traditional planning problems is to allow for a way to handle the identification of model updates and to account for the effect of these model updates on the user's expectation. 
A good way to go about this would be by acknowledging that that if the observer is actually watching the agent executing a plan, these explanations can delivered through and hence modeled as communicative or {\em explanatory actions}. 
These actions can, in fact, be seen as actions with epistemic effects in as much as they are aimed towards modifying the human mental model. 
This means that a solution to an EA planning problem can be seen as {\em self-explaining plans}, in the sense that some of the actions in the plan are aimed at helping people better understand the rest of the plan.

This puts EA planning squarely in the purview of epistemic planning, but the additional constraints enforced by the setting allow us to leverage relatively efficient methods to solve the problem at hand. These constraints include facts like: all epistemic actions are public, modal depth is restricted to one, modal operators only applied to literals, for any literal the observer believes it to be true or false and the robot is fully aware of all of the observer beliefs. 

Model updates in the form of epistemic effects of communication actions also open up the possibility of other actions having epistemic {\em side effects}. The definition of EA makes no claims as to how the model update information is delivered. It is quite possible that actions that the agent is performing to achieve the goal (henceforth referred to as task-level actions to differentiate it from primary epistemic communication actions) itself could have epistemic side-effects. This is something people leverage to simplify communication -- e.g. one might avoid providing prior description of some skill they are about to use when they can simply demonstrate it. So one of our goals with the compilation is to allow for such epistemic side effects; a factor that has previously been not considered in any of the earlier works. 
This consideration also enables us to also capture task level constraints that may be imposed on the communication actions.

\subsubsection{Compilation to classical planning.}
To support such self-explaining plans, we adopt a formulation that is similar to the one introduced in \cite{muise-epist} to compile reasoning about epistemic states into a classical planning problem. In our setting, each explanatory action can be viewed as an action with epistemic effects. One interesting distinction to make here is that the mental model now not only includes the human's belief about the task state but also their belief about the robot's model. This means that the planning model will need to separately keep track of (1) the current robot state, (2) the human's belief regarding the current state, (3) how actions would effect each of these (as humans may have differing expectations about the effects of each action) and (4) how those expectations change with explanations.

Given the model reconciliation planning problem $\Psi = \langle \mathcal{M}_{R}, \mathcal{M}_{H}\rangle$, we will generate a new planning model $\mathcal{M}_{\Psi} = \langle F_{\Psi},A_{\Psi},I_{\Psi},G_{\Psi}, C_{\Psi} \rangle$ as follows $F_{\Psi} = F \cup F_{\mathcal{B}} \cup F_{\mu} \cup \{ \mathcal{G}, \mathcal{I}\} $, where $F_\mathcal{B}$ is a set of new fluents that will be used to capture the human's belief about the task state and $F_{\mu}$ is a set of meta fluents that we will use to capture the effects of explanatory actions and $\mathcal{G}$ and $\mathcal{I}$ are special goal and initial state propositions. 
We will use the notation $\mathcal{B}(p)$ to capture the human's belief about the fluent $p$. We are able to use a single fluent to capture the human belief for each (as opposed to introducing two new fluents $\mathcal{B}(p)$ and $\mathcal{B}(\neg p)$) as we are specifically dealing with a scenario where the human's belief about the robot model is fully known and human either believes each of the fluent to be true or false. In this case, we also do not require any of the additional rules that were employed in \cite{muise-epist} to ensure that the state captures the deductive closure of the agent beliefs.

%\note{About observability of believed model}\\
% $F_{\mathcal{B}}$ would contain a new propositional fluent for each element of the original set $F$, i.e, for any $p \in F$, we will have a $\mathcal{B}(p) \in F_{\mathcal{B}}$. 

$F_{\mu}$ will contain an element for every part of the human model that can be changed by the robot through explanations. A meta fluent corresponding to a literal $\phi$ from the precondition of an action $a$ takes the form of $\mu^{+}({\phi}^{\textrm{prec}^a})$, where the superscript $+$ refers to the fact that the clause $\phi$ is part the precondition of the action $a$ in the robot model (for cases where the fluent represents an incorrect human belief we will be using the superscript $-$).

For every action $a = \langle \textrm{prec}^a, \textrm{adds}^a, \textrm{dels}^a\rangle \in A_R$ and its human counterpart $a_h = \langle \textrm{prec}^{a_h}, \textrm{adds}^{a_h}, \textrm{dels}^{a_h} \rangle \in A_H$, we define a new action $a_{\Psi} = \langle \textrm{prec}^{a_{\Psi}}, \textrm{adds}^{a_{\Psi}}, \textrm{dels}^{a_{\Psi}}\rangle  \in \mathcal{M}_{\Psi}$ whose precondition is given as:
\begin{multline*}
\textrm{prec}^{a_{\Psi}} = \textrm{prec}^{a_R} \cup \{\mu^{+}({\phi}^{\textrm{prec}^a}) \rightarrow \mathcal{B}(\phi) | \phi \in \textrm{prec}^{a_R}\setminus \textrm{prec}^{a_H}\} \\ \cup 
\{\mu^{-}({\phi}^{\textrm{prec}^a}) \rightarrow \mathcal{B}(\phi) | \phi \in \textrm{prec}^{a_H}\setminus \textrm{prec}^{a_R}\}
\\\cup \{\mathcal{B}(\phi) | \phi \in \textrm{prec}^{a_H}\cap \textrm{prec}^{a_R}\}
\end{multline*}
% \begin{multline*}
% \textrm{prec}^{a_{\mathcal{E}}} = \textrm{prec}^{a_R} \cup \\ \{\mu^{+}({\phi}^{\textrm{prec}^a}) \rightarrow \mathcal{B}(\phi) | \phi \in \textrm{prec}^{a_H}\setminus \textrm{prec}^{a_R}\} \cup\\ \{\mathcal{B}(\phi) | \phi \in \textrm{prec}^{a_H}\cap \textrm{prec}^{a_R}\}
% \end{multline*}
The important point to note here is that at any given state, an action in the augmented model is only applicable if the action is executable in robot model and the human believes the action to be executable. Unlike the executability of the action in the robot model (captured through unconditional preconditions) the human's beliefs about the action executability can be manipulated by turning the meta fluents on and off.
The effects of these actions can also be defined similarly by conditioning them on the relevant meta fluent.
In addition to these task level actions (represented by the set $A_{\tau}$), we can also define explanatory actions ($A_{\mu}$) that either add $\mu^+(*)$ fluents or delete $\mu^-(*)$. 
%A domain writer could also further build upon this augmented model by adding explanatory effects to existing task level actions or merge explanatory actions to some existing task level actions.

Special actions $a_{0}$ and $a_{\infty}$ that are responsible for setting all the initial state conditions true and checking the goal conditions are also added into the domain model. $a_{0}$ has a single precondition that checks for $\mathcal{I}$ and has the following add and delete effects:
\begin{multline*}
%\hspace{-20pt}
\textrm{adds}^{a_0} = \{\top\rightarrow p\mid p \in I_{R}\} \cup \{\top\rightarrow \mathcal{B}(p)\mid p \in I_{H}\} \\ \cup \{\top\rightarrow p \mid p \in F_{\mu^{-}} \}
\end{multline*}
\indent$\textrm{dels}^{a_0} = \{\mathcal{I}\}$

% $\textrm{prec}^{a_{\mathcal{E}}} = \textrm{prec}^{a_R} \cup \{\mu^{+}({\phi}^{\textrm{prec}^a}) \rightarrow \phi | \phi \in \textrm{prec}^{a_H}\setminus \textrm{prec}^{a_R}\}$
\vspace{5pt}
\noindent where $F_{\mu^{-}}$ is the subset of $F_{\mu}$ that consists of all the fluents of the form $\mu^{-}(*)$. Similarly, the precondition of action $a_{\infty}$ is set using the original goal and adds the proposition $\mathcal{G}$.
%\vspace{-10pt}
\begin{multline*}
\textrm{prec}^{a_{\infty}} = G_{R} \cup \{\mu^{+}({p}^{G}) \rightarrow \mathcal{B}(p) \mid p \in G_{R}\setminus G_{H}\} \cup \\
\{ \mu^{-}(p^G) \rightarrow \mathcal{B}(p) \mid p \in G_{H}\setminus G_{R}\} \cup
 \{\mathcal{B}(p) \mid G_{H}\cap G_{R}\}
\end{multline*}

% \vspace{10pt}
Finally the new initial state and the goal specification becomes $I_{\mathcal{E}} = \{\mathcal{I}\}$ and $G_{\mathcal{E}} = \{\mathcal{G}\}$ respectively. To see how such a compilation would look in practice, consider an action {\small$\textsf{(move\_from p1 p2)}$} that allows the robot to move from point {\small$\textsf{p1}$} to {\small$\textsf{p2}$} only if the path is clear. The action is defined as follows in the robot model:
{\small \begin{lstlisting}[mathescape]
(:action move_from_p1_p2
    :precondition (and (at_p1) (clear_p1_p2))
    :effect (and (not (at_p1)) (at_p2) ))
\end{lstlisting}}

Let us assume the human is aware of this action but does not care about the status of the path (as they assume the robot can move through any debris filled path). In this case, the corresponding action in the augmented model and the relevant explanatory action will be: 
{\small
\begin{lstlisting}[mathescape]
(:action move_from_p1_p2
    :precondition
    (and (at_p1) ($\mathcal{B}$((at_p1))) (clear_p1_p2)
       (implies 
         ($\mu^{+}_{prec}$(move_from_p1_p2, (clear_p1_p2))) 
         ($\mathcal{B}$((clear_p1_p2))) ))
    :effect (and  (not (at_p1)) (at_p2) 
           (not $\mathcal{B}$(at_p1))  $\mathcal{B}$(at_p2)) ))
       
(:action explain_$\mu^{+}_{prec}$_move_from_clear 
    :precondition (and )
    :effect (and $\mu^{+}_{prec}$(move_from_p1_p2, 
                 (clear_p1_p2)) ))
\end{lstlisting}
}
Finally $C_{\Psi}$ captures the cost of all explanatory and task level actions. For now we will assume that the cost of task-level actions are set to the original action cost in either robot or human model and the explanatory action costs are set according to $C_E$. Later, we will discuss how we can adjust the explanatory action costs to generate desired behavior.

We will refer to an augmented model that contains an explanatory action for each possible model updates and has no actions with effects on both the human's mental model and the task level states as the {\em canonical augmented model}. 

%These canonical augmented models would suffice for us to reproduce the type of explanations studied in \cite{explain}. 
%In reality, the domain writer could also attach specific explanatory effects to existing robot actions that gets automatically included in the augmented models.

Given an augmented model, let $\pi_{\mathcal{E}}$ be a plan that is valid for this model ($\pi_{\mathcal{E}}(I_{\Psi}) \subseteq G_{\Psi}$). 
From $\pi_{\mathcal{E}}$, we extract two types of information -- 
the model updates induced by the actions in the plan (represented as $\mathcal{E}(\pi_{\mathcal{E}})$) and the sequence of actions that have some effect of the task state represented as $\mathcal{D}(\pi_{\mathcal{E}})$ (we refer to the output of $\mathcal{D}$ as the task level fragment of the original plan $\pi_{\mathcal{E}})$). 
$\mathcal{E}(\pi_{\mathcal{E}})$ may contain effects from action in $\mathcal{D}(\pi_{\mathcal{E}})$. This brings us to our next theorem.
% first assertion of this paper, namely
%\\
%\note{Sarath: The rest of the section here is new}

\begin{theorem}
For a given EA problem $\Psi = \langle \mathcal{M}_R, \mathcal{M}_H\rangle$ the corresponding augmented model $\mathcal{M}_{\Psi}$ is a sound and complete formulation: (1) for every valid $\pi$ for $\mathcal{M}_{\Psi}$ the tuple $\langle \mathcal{E}(\pi), \mathcal{D}(\pi)\rangle$ is a valid solution for $\Psi$ and (2) for every valid solution $\langle \mathcal{E}^{\Psi}, \pi\rangle$, there exists a corresponding valid plan for $\pi'$ for $\mathcal{M}_{\Psi}$ such that $\mathcal{D}(\pi') = \pi$ and $\mathcal{E}(\pi') = \mathcal{E}^{\Psi}$. 
\end{theorem}
The proof sketch is included in the supplementary file.
%\begin{proof}[Proof Sketch]
%he soundness of plans generated from $\mathcal{M}_\Psi$ are guaranteed by the construction of the model as all the preconditions of the actions in the updated user model have to be met in the current plan. To see why the formulation is complete, consider a solution $<\mathcal{E}^{\Psi}, \pi>$ for $\Psi$. From the procedure for constructing $\mathcal{M}_{\Psi}$ we know that there must exist an explanatory action for each possible model difference. This means that there should exist a sequence of explanatory actions $\langle a_1,..,a_k\rangle$ that results in the same model updates captured by $\mathcal{E}^{\Psi}$. It is easy to see that $\langle a_1,..,a_k\rangle + \pi $ is a valid plan for $\mathcal{M}_{\Psi}$ hence proving the assertion.
%\end{proof}

The planner can automatically find positions of the explanatory actions, but to avoid any confusion that may arise from belief revisions on the users' end, we can enforce some common sense ordering like making any explanation related to an action to appear before the first instance of that action. 
This ordering will make sure that users are not confused about earlier action effects and also helps reduce branching, making planning more efficient.

\subsubsection*{Stage of Interaction and Epistemic Side Effects:}
One of the important parameters of the problem setting that we have yet to discuss is whether the explanation is meant for a plan that is proposed by the system (i.e the system presents a sequence of actions to the user) or are we explaining some plan that is being executed either in the real world or some simulation the user (observer) has access to. Even though the above formulation can be directly used for both scenarios, we can use the fact that the human is observing the execution of the plans to simplify the explanatory behavior by leveraging the fact that many of these actions may have epistemic side effects. This allows us to not explain any of the effects of the actions that the human can observe (for those effects we can directly update the believed value of the corresponding state fluent and the meta-fluent).\footnote{This means that when the plan is being executed, the problem definition should include the observation model of the human (which we assume to be deterministic). To keep the formulation simple, we ignore this for now. Including this additional consideration is straightforward for deterministic sensor models.}
This is beyond the capability of any of the existing algorithms in this
space of the explicability-explanation dichotomy.

This consideration also allows for the incorporation of more complicated epistemic side-effects wherein the user may infer facts about the task that may not be directly tied to the effects of actions. 
Such effects may be specified by domain experts or generated using heuristics.
Once identified, adding them to the model is relatively straightforward as we can directly add the corresponding meta fluent into the effects of the relevant action. An example for a simple heuristic would be to assume that the firing of a conditional effect results in the human believing the condition to be true. For example, if we assume that the robot had an action {\small $\textsf{(open\_door\_d1\_p3)}$} that had a conditional effect: 

{\small \begin{lstlisting}[mathescape]
(when (and (unlocked_d1)) (open_d1))
\end{lstlisting}}

Then in the compiled model, we can add a new effect:

{\small \begin{lstlisting}[mathescape]
(when (and (unlocked_d1)) 
    (and $\mathcal{B}$(open_d1) $\mathcal{B}$(unlocked_d1)))
\end{lstlisting}}
Even in this simple case, it may be useful to restrict the rule to cases where the effect is conditioned on previously unused fluents so the robot does not expect the observer to be capable of regressing over the entire plan.

\section{Optimality of the Agent}
\label{balance}

The compilation explored so far only takes into consideration the expectations the agent has about the safety of the plans (i.e the user would expect any plans generated to be valid and executable) and does not account for the user's expectation on whether the agent should act optimally. In the earlier example, if the agent just followed the plan that takes the robot through P5 and P6 with a {\small $\textsf{clear\_passage\_P5\_P6}$} action with no additional explanatory actions then the user may still be confused why the agent does not just follow the plan that involves going through P16 to P17 that it believes to be cheaper (marked in grey in the human's map).

Even in cases where the action costs are the same for the agent and the human, we cannot account for such expectations by merely generating optimal plans in the augmented model. For example, the optimal plan in the augmented model would be the one through P2 and P3 (the full plan is marked in blue in the robot map) with one extra explanatory action {\small $\textsf{explain\_$\mu^+_{I}$\_clear\_P2\_P3}$}.
While the above plan provides an explanation to ensure validity, ensuring the optimality of the resultant plan would require the agent to also explain that the passage from P16 to P17 is blocked, which would clearly be more expensive than choosing the valid plan for any non-zero cost for explanatory actions.

This means that in order to accommodate such considerations we need to go beyond the compilation discussed so far. One approach would be to prune
all solutions where the task level fragment of the plan ($\mathcal{D}(\pi)$) is suboptimal in the updated human model. A simple way to enforce this would be to extend the planner to perform an optimality test for the current plan during the goal test. It may be possible to use more intelligent pruning to reduce the number of goal tests (e.g. one could leverage the fact that the optimality test never needs to be repeated for the same set of model updates) and we could design heuristics that take into account optimality aspects. In this paper, we adopt this simple approach as a first step towards modeling these novel behaviors.

\subsection{Balanced Plans \textit{vs.} Agent Optimal Plans}

Even when generating plans that preserve the user's expectations about agent optimality, the agent could generate two types of plans: agent optimal plans \cite{explain} or balanced plans \cite{balance}. 
In the first scheme, the agent chooses to select self-explanatory plans whose task level fragment is going to be optimal in the original agent model and then choose the minimal explanations that justifies the optimality plan (i.e the plan is optimal in the user's updated model). 
Such explanations are referred to as Minimally Complete Explanation or MCE 
(the agent could also choose among the optimal plans the one that requires the cheapest MCE). An example would be choosing the plan highlighted in blue in robot model and then explaining that the P2 to P3 is clear and P16 to P17 is blocked. In the latter scheme, the agent could choose plans that are easiest to explain (here again we need to ensure that after the explanation the plan is optimal in the updated model). 
For example, in the USAR scenario if communication is expensive, it may be easier to choose the plan to move through P5 and P6 with a clear passage action since we only need to explain that the passage P16 to P17 is blocked.

In the first case, the agent is effectively prioritizing any loss of optimality over any overhead accrued by communicating the explanation, while in the latter case the agent accounts for the cost of both the plan it is performing and the explanation cost (the cost of communication and possibly the computational overhead experienced by the user on receiving the explanation). 
By assigning explanatory costs to explanatory actions we are essentially generating balanced plans but there may be scenarios where the agent needs to stick to its optimal plan. 
We can generate such agent optimal plans by setting lower explanatory action costs. Before we formally state the bounds for explanatory costs, 
let us define the concept of {\em optimality delta} 
%(\note{Am I reinventing a concept that already exists in planning?})
(denoted as $\Delta\pi_{\mathcal{M}}$) for a planning model, which captures the cost difference between the optimal plan and the second most optimal plan. More formally $\Delta\pi_{\mathcal{M}}$ can be specified as:
{\small
\begin{multline*}
\Delta\pi_{\mathcal{M}} = \textrm{max}\{ v \mid v \in \mathbb{R}~\wedge \not \exists \pi_1,\pi_2((0 < (C(\pi_1) - C(\pi_2)) < v)\\ \wedge \pi_1(I_{\mathcal{M}}) \models_{\mathcal{M}} G_{\mathcal{M}}  \wedge \pi_2(I_{\mathcal{M}}) \in \Pi^{*}_{\mathcal{M}}\}%\models_{\mathcal{M}} G_{\mathcal{M}}   \}
\end{multline*}
}
\begin{theorem}
\label{MCE_THEOR}
In a canonical augmented model $\mathcal{M}_{\Psi}$ for an EA planning problem $\Psi$, if the sum of costs of all explanatory actions is $\leq \Delta\pi_{\mathcal{M}_R}$ and if $\pi$ is the cheapest valid plan for $\mathcal{M}_{\Psi}$ such that $\mathcal{D}(\pi) \in \Pi^{*}_{\mathcal{M}_{\Psi} + \mathcal{E}(\pi)}$, then:

\begin{itemize}
\item[(1)]
$\mathcal{D}(\pi)$ is optimal for $\mathcal{M}_R$
\item[(2)]
$\mathcal{E}(\pi)$ is the MCE for $\mathcal{D}(\pi)$
\item[(3)]
There exists no plan $\hat{\pi} \in \Pi^*_R$ such that MCE for $\mathcal{D}(\hat{\pi})$ is cheaper than $\mathcal{E}(\pi)$, 
i.e. the search will find an the plan with the smallest MCE.  
\end{itemize}
\end{theorem}

%\begin{proof}[Proof Sketch]
%We observe that there exists no valid plan $\pi'$ for the augmented model ($\mathcal{M}_{\Psi}$) with a cost lower than that of $\pi$ and where the task level fragment ($\mathcal{D}(\pi')$) is optimal for the human model. 
%Let's assume $\mathcal{D}(\pi) \not\in \Pi^*_{\mathcal{R}}$ (i.e current plan's task-level fragment is not optimal in robot model) and let $\hat{\pi} \in \Pi^*_{\mathcal{R}}$. Now let's consider a plan $\hat{\pi}_{\mathcal{E}}$ for augmented model that corresponds to the plan $\hat{\pi}$, i.e,  %$\mathcal{E}(\hat{\pi}_{\mathcal{E}})$ is the MCE for the plan $\hat{\pi}$ and $\mathcal{D}(\hat{\pi}_{\mathcal{E}}) = \hat{\pi}$. 
%Then the given augmented plan $\hat{\pi}_{\mathcal{E}}$ is a valid solution for our augmented planning problem $\mathcal{M}_{\Psi}$ (since the $\hat{\pi}_{\mathcal{E}}$ consists of the MCE for $\hat{\pi}$, the plan must be valid and optimal in the human model), moreover the cost of %$\hat{\pi}_{\mathcal{E}}$ must be lower than $\pi$. This contradicts our earlier assumption hence we can show that $\mathcal{D}(\pi)$ is in fact optimal for the robot model.

%Using a similar approach we can also show that no cheaper explanation exists for $\pi_{\mathcal{E}}$ and there exists no other plan with a cheaper explanation.
%\end{proof}

%The proof of the theorem can be found in the supplementary file (\url{https://goo.gl/qKABpi}).
The proof is included in the supplementary file.
Note that while it is hard to find the exact value of the optimality $\Delta\pi_{\mathcal{M}}$, it is guaranteed to be $\geq 1$ for domains with only unit cost actions or $\geq (C_2 - C_1)$, where $C_1$ is the cost of the cheapest action and $C_2$ is the cost of the second cheapest action, i.e. $\forall a (C_{\mathcal{M}}(a) < C_2 \rightarrow  C_{\mathcal{M}}(a) = C_1)$. Thus allowing us to easily scale the cost of the explanatory actions to meet this criteria.

\begin{table}[!tp]
\centering
\small
  \begin{tabular}{r|c|c|c|c}
    \cmidrule(lr){1-5}
    \multirow{3}{*}{} &
    \multicolumn{2}{c|}{\multirow{1}{*}{New Compilation}} &\multicolumn{2}{c}{\multirow{1}{*}{Model Space Search}}\\
        % \midline
        \cmidrule(lr){2-5}
       & {\bf coverage} & {\bf runtime} & {\bf coverage} & {\bf runtime}\\
%       %\\[1ex]
%       & &
%       &
%       & \multicolumn{3}{c|}{}
%       & \multicolumn{3}{c|}{}
%       & \multicolumn{3}{c|}{}
      %\\[1ex]
        \cmidrule(lr){1-5}
      %\midrule
      \multirow{1}{*}{Blocksworld} &13/15&{\bf 569.38}&13/15&2318.73\\
%       &84.07&7 &1&7.57&1&2.40&7.57&1&2.16&7.57&1&3.63   \\ 
%       &84&7&2&10.35&1.27&7.31&10.36&1.27&7.55&11.24&1.46&5.84\\
%       %&&&3&13.17&1.55&15.54&13.17&1.56&15.48&14.03&1.73&7.84\\
%       &90.7&7&4&17.33&1.81&26.29&17.34&1.81&24.73&18.27&1.96&10.36\\
      \hline
       \multirow{1}{*}{Elevator} &{\bf 15}/15&{\bf 59.20}&1/15&3382.462\\
       \hline
      \multirow{1}{*}{Gripper} &5/15&2301.90&{\bf 6}/15&{\bf 2093.54}\\
       \hline
\multirow{1}{*}{Driverlog} &{\bf 4}/15&{\bf 2740.38}&2/15&3158.59\\
\hline
\multirow{1}{*}{Satellite} &{\bf 2}/15&{\bf 3186.93}&0/15&3600\\
%       \multirow{3}{*}{Woodworking} &&&&&&\\ &&&&&&\\ &&&&&&\\
%       \hline
%       \multirow{3}{*}{Sokoban}  &&&&&&\\ &&&&&&\\ &&&&&&\\
%     \hline
    \bottomrule
  \end{tabular}
  %\raisebox{-0.5\height}{\includegraphics[width=.30\textwidth]{IMAGES/chart_time_gain}}
    \caption{\small{
        Coverage and average runtime (sec) for explanations generated for a few standard IPC domains.}}
\label{tab1}
\end{table}

\section{Evaluation}
\label{emp}
Since the nature of our solution has already been validated in literature through human factors evaluation -- model reconciliation explanation has been studied in \cite{chakraborti2019plan}, balanced plans in \cite{balance}, explicable plans in \cite{exp-yu,exp-anagha}, and the use of physical actions to communicate robot model information in \cite{kwon2018expressing} -- 
%we will study the empirical properties of our compilation instead.
we will focus on demonstrating the generality of our framework and studying empirically the performance of the compilation.

\subsection{Illustrative Example of Cost-Tradeoff}

We start by demonstrating how our approach can lead to different solution by altering various costs associated with agent actions. 
Consider again the USAR domain described earlier: the models for the robot and the user is provided in the supplementary (the action for opening a door has an epistemic side effect that the observer would know that the door is unlocked).
We start by assigning a cost of 10 to every robot action other than clear-rubble action (which is 50) and the move-through-door action (set to 20). We set the cost of communication action to 1 to start with. 
The solution produced corresponds to the blue plan in Figure \ref{fig:1}.
{\small\begin{lstlisting}[mathescape]
explains_$\mu_{init}^{+}$_clear_p2_p3-> 
explains_$\mu_{init}^{-}$_clear_p16_p17-> 
move_p1_p2-> move_p2_p3-> move_p3_p4-> 
move_p4_p11->move_p11_p13-> move_p13_p14
-> move_p14_p18-> move_p18_p17
\end{lstlisting}}
This plan includes the \textbf{optimal robot plan and corresponding MCE}. 
Now if we were to set the cost of communication actions to 100, we see the agent deviating to plans which on their own may not be optimal -- e.g.
a plan that involves opening the door at P8:
{\small
\begin{lstlisting}[mathescape]
explains_$\mu_{init}^{-}$_clear_p16_p17-> 
move_p1_p7-> move_p7_p8-> opendoor_p8_d1->
movethroughdoor_p8_p9_d1-> move_p9_p10-> 
move_p10_p13-> move_p13_p14-> move_p14_p18-> 
move_p18_p17
\end{lstlisting}}
Here the robot does not have to explicitly provide a separate explanation for the status of the door, but still needs to explain that the path from P15 to P16 is blocked. Note that this plan is an example of \textbf{a balanced plan} that leverages \textbf{epistemic side effects}.

Now we go one step further and relax the need to assure optimality of the plan in the human model from a hard constraint to just a penalty (details of this extension are part of the supplementary). 
This gets us the exact same plan as above but without the explanation about the blocked corridor from P15 to P16, thus allowing a notion
of soft explicability.

\subsection{Runtime Complexity}

Next we establish how our approach compares in terms of runtime to previous work. %We will consider the approaches discussed inas a point of comparison. 
In particular, we will use as reference the optimistic and approximate version of the balancing approach in \cite{balance} that identifies only one optimal plan per search node and the search ends as soon as it finds a node where the optimal plan produced has the same cost as the robot plan and is executable in the robot model. This means all the solutions we generate are guaranteed to be better (in terms of cost) than that generated by the other.
For comparison, we selected five IPC domains and for each domain, we created three unique models by introducing 10 random updates in the model, except in the case of Gripper and Driverlog where only 5 were removed. 
Each of these five domains were paired with five problem instances and then tested on each of the possible configurations. Each instance was run with a limit of 30 minutes, all explanatory actions were restricted to the beginning of the plan and the cost of explanatory actions were set to be twice the cost of original action. Figure \ref{tab1} lists the time taken to solve each of these problems. For calculating the average runtime, we used 1800 secs as the stand in for the runtime of all the instances that timed out. We used h\_max
(admissible) as the heuristic for all the configurations. 
% We chose to focus on h\_max since it was an admissible heuristic that was easy to implement and allowed to use the more expressive planning fragments without much issues.
% All tests were run on a Linux machine with 64GB memory with a python based planner with the $h_{max}$ heuristic.

As clearly apparent from the table, the new approach does better than the original method for generating balanced plans for most of the domains. 
Gripper seems to be the only domain, where model search seem to be doing slightly better but this is also a domain that had the smallest number of model differences. 
This indicates that the ability to leverage planning heuristics can make a marked difference in domains with a large number of explanatory actions. 

\section{Related Work}
\label{related}

We end with a review of existing literature and emphasize key 
differentiators with the proposed work.

\subsubsection*{Epistemic Planning} 
It is well understood in social sciences that explanations must be generated while keeping in mind the beliefs of the agent receiving the explanation \cite{miller}. 
As such, epistemic planning makes for an excellent framework for studying the problem of generating these explanations. 
While the most general formulation of epistemic planning has been shown to be undecidable, many simpler fragments have been identified \cite{bolander2015complexity}. 
In human-aware planning settings too, there has been increasing consensus that epistemic planning could be an extremely useful tool. 
Readers can refer to \cite{sociomill} for an overview of works done in employing epistemic planning for ``social planning''.
Recently, there have been a lot of interest in developing efficient methods for planning in such settings \cite{muise-epist,kominis2015beliefs,kominis2017multiagent,le2018efp,huang2018general}.

\subsubsection*{Model Reconciliation} 
Among the works related to model reconciliation, the work that is most closely connected to ours is \cite{balance}. 
The idea of balanced plans were first proposed in that work. 
Unfortunately, the actual algorithm they study is incomplete and is not guaranteed to produce the least expensive balanced plan. 
Even the complete version they hypothesize in their paper relies enumerating all the possible optimal plans for a given updated model, which can be extremely inefficient, particularly since they expect to perform this for every possible model in the model space. As we see in the empirical evaluations, 
our method (which is also complete) is often faster against their optimistic approximate version. Moreover the methods discussed in that paper are unable to utilize task-level actions with epistemic side effect or take into account task level constraints for purely communicative actions and the effects of execution on an observer, as we illustrate through examples.

\subsubsection*{Communicative Actions} 
Our work also looks at the use of explanatory actions as a means of communicating information to the human observer. 
The most obvious types of such explanatory action includes purely communicative actions such as speech \cite{tellex2014asking} or the use of mixed reality projections \cite{iros-proj,ganesan2017mediating}. 
Recent works have shown that physical agents could also use movements to relay information such as intention \cite{macnally2018action,dragan2013legibility} and incapability \cite{kwon2018expressing}. Our framework 
% could be easily adopted to any of these explanatory actions and would 
allows for a natural trade-off between these different types of communication.
%\vspace{-1pt}

\subsubsection*{Contrastive Explanations and Inferential Capabilities} 
%\subsubsection*{Contrastive Explanations} 
Many recent works dealing with explanation generation for planning
have looked at characterizing explanations in terms of the types of questions they answer \cite{danmaga,smith2012planning}. 
This characterization is orthogonal to the question of what type of information constitutes valid explanations. 
Putting aside questions regarding observability, the reason why a user 
may request an explanation is either due to knowledge mismatch (incomplete or incorrect knowledge of the task) or due to limitations of their inferential capabilities. The answer to any of these questions
would require correcting the human's model of the task and/or providing inferential assistance. Works that have looked at model reconciliation explanations have mostly focused on the former. 
Explanations discussed in this paper can be viewed as an answer to the question {\em ``Why this plan?''} (which can also be viewed as a contrastive question of the form {\em ``Why this plan and not any other plan?''}). 
This is not to say that in complex scenarios just the model reconciliation information would suffice, but it would need to be supplemented with information internal to the model that can address the differences in inferential capabilities. 
Use of abstractions \cite{abs-ijcai}, providing refutation of specific foils \cite{abs-ijcai} and providing causal explanations \cite{seegebarth} could be used to augment model reconciliation.% explanations.

\section{Conclusion}
\label{concl}

The paper presents a unifying formulation for the task of planning in the presence of users with incorrect mental models.
The formulation allows us to unify, for the first time, explanatory and explicable paradigms into a single framework that is still amenable to classical planning.
We discuss how this formulation can be extended to capture novel explanatory behaviors hitherto unexplored in literature 
while being computationally more efficient than methods that rely on direct model space search. 
One of the exciting features of our work is that we are able to place \exact\ within the realm of epistemic planning, thereby laying the ground work to study more complex interaction scenarios including cases with more levels of nesting, uncertainty about mental models, more expressive models, incorporating non-deterministic effects, and so on.
It would also be worth investigating specific considerations for
choosing heuristics or formulating new ones for such problems.
%\vspace{-1pt}

%\clearpage
\bibliographystyle{named}
\bibliography{aaai_subm}
\end{document}